\documentclass[conference]{IEEEtran}

\usepackage[square,sort,compress,comma, numbers]{natbib}

\usepackage{balance}

\usepackage{amsmath,amssymb,amsfonts, amsthm}
\usepackage{mathtools}
\usepackage{algorithmic}
\usepackage{setspace}
\usepackage{graphicx}
\usepackage{textcomp}
\usepackage{xcolor}

\usepackage{blindtext}
\usepackage{tikz-qtree}
\usepackage{forest}
\usepackage{hyperref}
\usepackage{csquotes}
\usepackage{booktabs}
\usepackage[caption=false]{subfig}
\usepackage{enumitem}
\usepackage[ruled, lined, linesnumbered, titlenumbered]{algorithm2e}
\newcommand{\InOut}{\SetKwInOut{Input}{Input}\SetKwInOut{Output}{Output}}
\SetKwComment{Comment}{$\triangleright$\ }{}

\usepackage{nicefrac}
\usepackage{threeparttable}
\usepackage{etoolbox}
\appto\TPTnoteSettings{\footnotesize}
\usepackage{adjustbox}

\usepackage{siunitx}

\usepackage{xparse}
\usepackage{forloop}
\newcounter{loopc}
\NewDocumentCommand\towrite{O{1}m}%
  {{\color{red}#2\forloop{loopc}{1}{\value{loopc} < #1}{; #2}}}
 
\DeclareMathOperator*{\argmax}{arg\,max}
\DeclareMathOperator*{\argmin}{arg\,min}

\def\BibTeX{{\rm B\kern-.05em{\sc i\kern-.025em b}\kern-.08em
    T\kern-.1667em\lower.7ex\hbox{E}\kern-.125emX}}

\usepackage{acronym}
\acrodef{RAPID}[RAPID]{\underline{R}educing s\underline{A}mples by \underline{P}runing of \underline{I}nlier \underline{D}ensities}
\newcommand{\OurSamplingMethod}{\acs{RAPID}\xspace}

\newcommand*{\constraintref}[1]{\hyperref[#1]{Constraint~\ref*{#1}}}
\newcommand*{\inequalityref}[1]{\hyperref[#1]{Inequality~\ref*{#1}}}
\newcommand*{\opref}[1]{\hyperref[#1]{Optimization Problem~\ref*{#1}}}
\newcommand{\website}{\url{https://www.ipd.kit.edu/ocs}\xspace}

\newtheorem{definition}{Definition}
\newtheorem{characteristic}{Characteristic}
\newtheorem{theorem}{Theorem}
\newtheorem{lemma}{Lemma}

\theoremstyle{definition}
\newtheorem{remark}{Remark}

\begin{document}
\bstctlcite{IEEEexample:BSTcontrol}

\title{Efficient SVDD Sampling with Approximation Guarantees for the Decision Boundary}

\makeatletter
\newcommand{\linebreakand}{%
	\end{@IEEEauthorhalign}
	\hfill\mbox{}\par
	\mbox{}\hfill\begin{@IEEEauthorhalign}
}
\makeatother

\author{\IEEEauthorblockN{Adrian Englhardt}
\IEEEauthorblockA{\textit{Karlsruhe Institute of Technology} \\
adrian.englhardt@kit.edu}
\and
\IEEEauthorblockN{Holger Trittenbach}
\IEEEauthorblockA{\textit{Karlsruhe Institute of Technology} \\
holger.trittenbach@kit.edu}
\and
\IEEEauthorblockN{Daniel Kottke}
\IEEEauthorblockA{\textit{University of Kassel} \\
daniel.kottke@uni-kassel.de}
\linebreakand
\IEEEauthorblockN{Bernhard Sick}
\IEEEauthorblockA{\textit{University of Kassel} \\
bsick@uni-kassel.de}
\and
\IEEEauthorblockN{Klemens B\"ohm}
\IEEEauthorblockA{\textit{Karlsruhe Institute of Technology} \\
klemens.boehm@kit.edu}
}

\maketitle
\thispagestyle{plain}
\pagestyle{plain}

\begin{abstract}

Support Vector Data Description (SVDD) is a popular one-class classifiers for anomaly and novelty detection.
But despite its effectiveness, SVDD does not scale well with data size.
To avoid prohibitive training times, sampling methods select small subsets of the training data on which SVDD trains a decision boundary hopefully equivalent to the one obtained on the full data set.
According to the literature, a good sample should therefore contain so-called boundary observations that SVDD would select as support vectors on the full data set.
However, non-boundary observations also are essential to not fragment contiguous inlier regions and avoid poor classification accuracy.
Other aspects, such as selecting a sufficiently representative sample, are important as well.
But existing sampling methods largely overlook them, resulting in poor classification accuracy.

In this article, we study how to select a sample considering these points.
Our approach is to frame SVDD sampling as an optimization problem, where constraints guarantee that sampling indeed approximates the original decision boundary.
We then propose \OurSamplingMethod, an efficient algorithm to solve this optimization problem.
\OurSamplingMethod does not require any tuning of parameters, is easy to implement and scales well to large data sets.
We evaluate our approach on real-world and synthetic data.
Our evaluation is the most comprehensive one for SVDD sampling so far.
Our results show that \OurSamplingMethod outperforms its competitors in classification accuracy, in sample size, and in runtime.

\end{abstract}


\section{Introduction}\label{sec:intro}

Support Vector Data Description (SVDD) is one of the most popular and actively researched one-class classifiers for anomaly and novelty detection~\cite{tax_support_2004, liu2010fast, trittenbach_overview_2018}.
The basic variant of SVDD is an unsupervised classifier that fits a tight hypersphere around the majority of observations, the inliers, to distinguish them from irregular observations, the outliers.
Despite its resounding success, a downside is that SVDD and its progeny do not scale well with data size~\cite{trittenbach2019validating}.
Even efficient solvers like decomposition methods~\cite{chaudhuri2018sampling, chu2004scaling, kim2007fast, platt1998sequential} result in training times prohibitive for many applications.
In these cases, sampling for data reduction is essential~\cite{li2011selecting, hu2014fast, li2019health, alam2020sample, sun2016heuristic, qu2019towards, li2018information, xiao2014k, zhu2014boundary, krawczyk2019instance}.

\begin{figure}[t]
	\includegraphics[width=0.9\columnwidth]{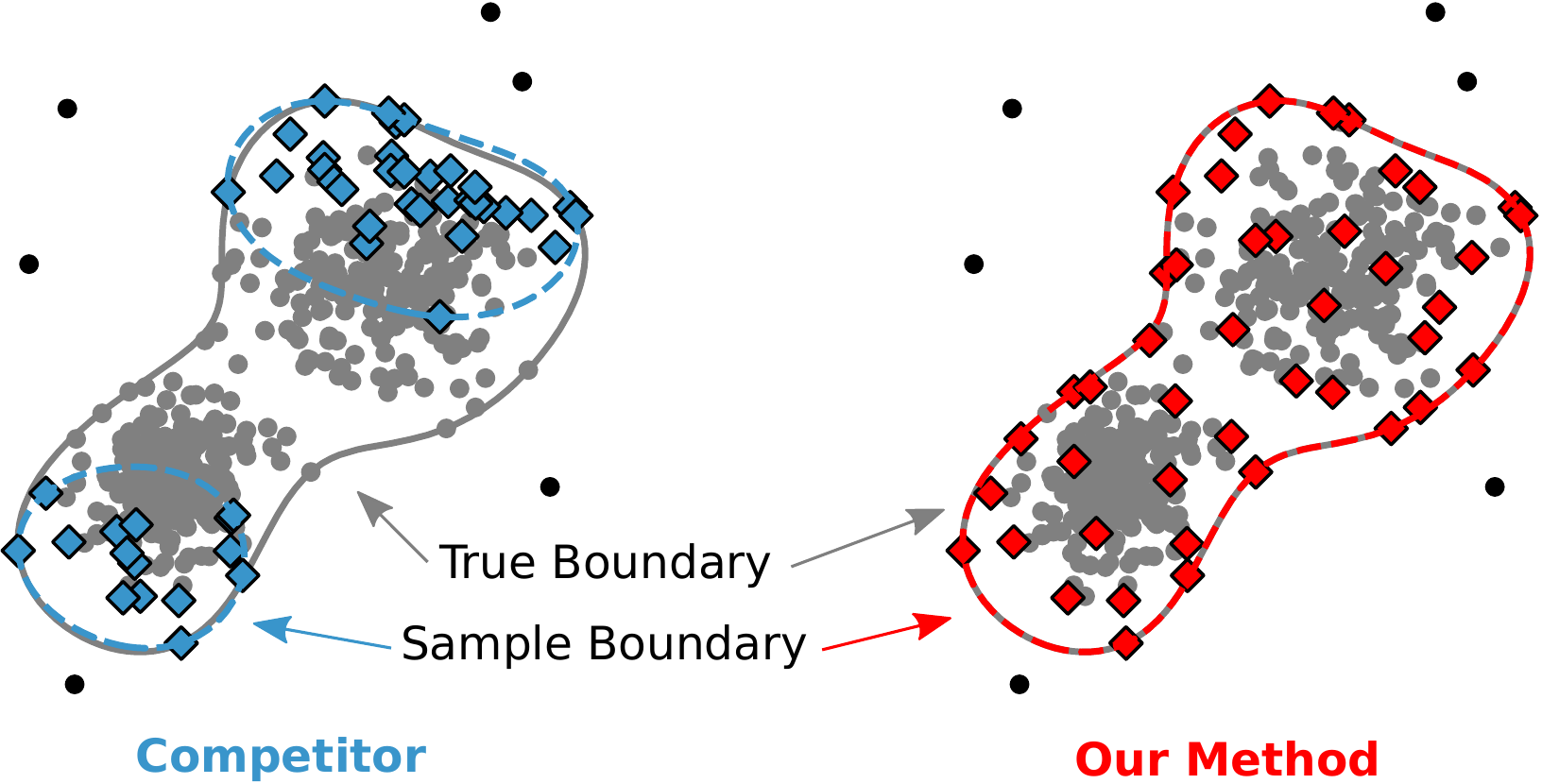}
	\centering
	\caption{Sample and decision boundary of a state-of-the-art boundary-point method~\cite{alam2020sample} and of our method \OurSamplingMethod.}
	\label{fig:intro-example}
\end{figure}

One of the defining characteristics of SVDD is that only a few observations, the support vectors, define a decision boundary.
Thus, a good sample is one for which SVDD selects support vectors similar to the original ones, i.e., the ones obtained on the full data set.
This has spurred the design of sampling methods that try to identify support-vector candidates in the original data, to retain them in the sample~\cite{li2011selecting, li2019health, qu2019towards, hu2014fast, alam2020sample, li2018information, xiao2014k, zhu2014boundary}.
A common approach is to select so-called \enquote{boundary points} as support-vector candidates, e.g., observations that are dissimilar to each other~\cite{li2011selecting, zhu2014boundary}.

But calibrating existing methods such that they indeed identify boundary points is difficult.
A reason is that the sample they return depends significantly on the choice of exogenous parameters, and selecting suitable parameter values is not intuitive (see \autoref{sec:experiments}).
A further shortcoming is that including all boundary points in a sample does not guarantee SVDD training to indeed yield the original support vectors.
The issue is that selection of support vectors hinges on other aspects, such as the ratio between inliers and outliers in the sample and a sufficient number of non-boundary observations in the sample.
Disregarding them may, for instance, fragment contiguous inlier regions and yield wrong outlier classifications after sampling, see~\autoref{fig:intro-example}. 
The influence of these aspects on SVDD is known, but their effects on sample selection are not well studied.
It is an open question how to select a sample where SVDD indeed approximates the original decision boundary.
Finally, a point largely orthogonal to these issues is that there also is very limited experimental comparison among competitors.
This makes an empirical selection of suitable SVDD sampling methods difficult as well.
\medskip

\textbf{Contributions.} In this article, we propose a novel way to SVDD sampling.
We make three contributions.
First, we reduce SVDD sampling to a decision-theoretic problem of separating data using empirical density values.
Based on this reduction, we formulate SVDD sampling as a constrained optimization problem.
Its objective is to find a minimal sample where the density of all observations of the data set is close-to-uniform.
We provide theoretical justification that a sample obtained in this way i) prevent a fragmentation of the inlier regions, and ii) retain the observations necessary to identify the original support vectors.

Second, we propose \ac{RAPID}, an efficient algorithm to solve the optimization. 
\OurSamplingMethod is the first SVDD sampling algorithm with theoretical guarantees on retaining the original decision boundaries.
\OurSamplingMethod does not require any parameters in addition to the ones already required by SVDD.
This lets \OurSamplingMethod stand out from existing methods, which all hinge on mostly unintuitive, exogenous parameters.
\OurSamplingMethod further is easy to implement, and scales well to very large data sets.

Third, we conduct the -- by far -- most comprehensive comparison of SVDD sampling methods. 
We compare \OurSamplingMethod against \num{8} methods on \num{21} real-world and \num{85} synthetic data sets.
In all experiments, \OurSamplingMethod consistently produces a small sample with high classification quality.
Overall, \OurSamplingMethod outperforms all of its competitors in the trade-off between algorithm runtime, sample size, and classification accuracy, often by an order of magnitude.
\section{Fundamentals}
\label{sec:fundamentals}

The objective of SVDD is to learn a description of a set of observations, the \emph{target}.
A good description allows to distinguish the target from other, non-target observations.
In our article, we focus on outlier detection where the targets are inliers, and the non-targets are outliers.
First, we discuss common assumptions for outlier detection.
We then introduce preliminaries and the SVDD optimization problem.

\paragraph{Assumptions}
A fundamental assumption is that observations from the target class come from a well-defined, albeit unknown distribution.
However, this may not hold for non-target observations, since outliers do not necessarily follow a common distribution.
Next, the non-target distribution may change, e.g., when novelties occur.
Thus, we assume that one can only estimate the target distribution.
A consequence is that binary classifiers are not applicable to outlier detection.

Next, one often makes assumptions regarding the composition of the training data.
In the \emph{target-only scenario}, all training observations come from the target class.
An example application would be novelty detection, i.e., novelties appear, per definition, only after training.
However, we focus on the \emph{outlier scenario} where the training data contains a majority of target observations and a few outliers.
Typically, the ratio of outliers to inliers in the data is unknown, but a common assumption is that domain experts can estimate it~\cite{achtert2010visual}.

\paragraph{Preliminaries}
Let $\mathbf{X} = \langle x_1, x_2, \dots, x_N \rangle$ be a data set of $N$ \emph{observations} from the domain $\mathbb{X} = \mathbb{R}^M$ where $M$ is the number of dimensions.
A \emph{sample} is a subset $\mathbf{S} \subseteq \mathbf{X}$ of the data set with sampling ratio $\nicefrac{\vert \mathbf{S} \vert}{N}$.
Further, we denote $x \in \mathbf{S}$ as \emph{selected}, and $x \notin \mathbf{S}$ as \emph{not-selected} observations.
The probability density of $\mathbf{X}$ is $p(x)$.
Further, let $\mathbf{Y} = \langle y_1, y_2, \dots, y_N \rangle$ be a ground truth, i.e., each entry is the realization of a dichotomous variable $\mathbb{Y} = \{\text{in}, \text{out}\}$.
The ground truth densities are the conditional probability densities $p_{\text{inlier}}(x) = P(\mathbf{X} = x \mid \mathbf{Y} = \text{in})$, and $p_{\text{outlier}}(x) = P(\mathbf{X} = x \mid \mathbf{Y} = \text{out})$ respectively.
One can estimate the empirical density of $\mathbf{X}$ by kernel density estimation.
\begin{equation}
    d_{\mathbf{X}}(x) = \sum_{x' \in \mathbf{X}} k(x, x')
\end{equation}
where $k$ is a kernel function with $k(x,x) = 1$.
A popular choice is the Gaussian kernel $k_\gamma(x, x') = e^{- \gamma \lVert x - x' \rVert}$, where $\gamma \geq 0$ is the parameter to control the kernel bandwidth.
We use the shorthand $d_x = d_{\mathbf{X}}(x)$ when the reference to $\mathbf{X}$ is unambiguous. 
Note that $d_{\mathbf{X}}$ requires normalization further to represent a probability density.
Densities can be used to characterize observations in different ways.
\begin{definition}[Level Set]\label{def:level-set}
    A level set is a set of observations with equal density $L_\theta \coloneqq \{x \in \mathbf{X} \colon d_x = \theta\}$.
    A super-level set is a set of observations with  $L^{+}_{\theta} \coloneqq \{x \in \mathbf{X} \colon d_x \geq \theta\}$.
\end{definition}
\noindent
One way to use level sets to categorize observations is to define a \emph{level-set classifier} as a function of type $g\colon \mathbb{X} \rightarrow \mathbb{Y}$ with
\begin{equation}\label{eq:density-decision-function}
	g^\mathbf{X}_{\theta}(x) = \begin{cases}
		\text{in} & \textit{if} \ x \in L^{+}_{\theta}\\
		\text{out} & \textit{else}.
	\end{cases}
\end{equation}
Another useful categorization is to separate observations into boundary points and inner points.
There are different ways to define a boundary of $\mathbf{X}$~\cite{li2011selecting, hu2014fast,li2019health, qu2019towards, alam2020sample, li2018information, xiao2014k, zhu2014boundary}.
For this article, we define boundary points as observations with density values close to the minimum empirical density.
\begin{definition}[Boundary Point]\label{def:boundary-point}
    Let $d_\text{min} = \min_{x \in \mathbf{X}} d_x$, and let $\delta$ be a small positive value.
    An observation $x \in \mathbf{X}$ is a boundary point of $\mathbf{X}$ if $x \in \mathbf{B}^\mathbf{X}$ with $\mathbf{B}^\mathbf{X} = L^{+}_{d_\text{min}} \setminus L^{+}_{(d_\text{min} + \delta)}$.
\end{definition}

\paragraph{SVDD Classifier}
SVDD~\cite{tax_support_2004} is a quadratic optimization problem that searches for a minimum enclosing hypersphere with center $a$ and radius $R$ around the data.
\begin{equation*}\label{eq:svdd-primal}
	\begin{aligned}
		\text{SVDD} \colon & \underset{a,\ R,\ \boldsymbol{\xi}}{\text{minimize}} & & R^2 + C \cdot \sum_{i=1}^{N} \xi_i \\
		& \text{subject to}	& & \lVert x_i - a \rVert^2 \leq R^2 + \xi_i, \; i = 1, \ldots, N \\
		& & & \xi_i \geq 0, \; i = 1, \ldots, N
	\end{aligned}
\end{equation*}
with cost parameter $C$ and slack variables $\boldsymbol{\xi}$.
Solving SVDD gives a fixed $a$ and $R$ and a decision function
\begin{equation}
f^\mathbf{X}(x) = \begin{cases}
\text{in} & \textit{if} \ \lVert x - a \rVert^2 \leq R^2 \\
\text{out} & \textit{else}.
\end{cases}
\end{equation}
When solving SVDD in the dual space, $f^\mathbf{X}$ only relies on inner product calculations between $x$ and some of the training observations, the support vectors.
So inference with SVDD is efficient if the number of support vectors is low.
Also note that under mild assumptions, SVDD is equivalent to $\nu$-SVM~\cite{Scholkopf2001-wh}.

SVDD has two hyperparameters, $C$ and a kernel function $k$.
$C \in \mathbb{R}_{[0,1]}$ is a trade-off parameter.
It allows some non-target observations in the training data to fall outside the hypersphere if this reduces the radius significantly.
Formally, observations outside the hypersphere with positive slack $\xi > 0$ are weighted by a cost $C$.
High values for $C$ make excluding observations expensive; based on the dual of SVDD, one can see that if $C=1$, SVDD degenerates to a hard-margin classifier~\cite{tax_support_2004}.

To allow decision boundaries of arbitrary shape, one can use the well-known kernel trick to replace inner products in the dual of SVDD by a kernel function $k$.
The most popular kernel with SVDD is the Gaussian kernel.
Its bandwidth parameter $\gamma$ controls the flexibility of the decision boundary.
For $\gamma \! \rightarrow \! 0$, the decision boundary in the data space approximates a hypersphere.
Choosing good values for the two hyperparameters $\gamma$ and $C$ is difficult~\cite{liao2018new}.
There is no established way of setting the parameter values, and one must choose one of the many heuristics to tune SVDD in an unsupervised setting~\cite{scott2015multivariate, liao2018new, tax_support_2004, trittenbach2019active}.
\section{Related Work}
\label{sec:related-work}

\begin{figure}[t]

	\forestset{
		child frame/.style n args={2}{
			tikz+={
				\node [draw=blue, dotted, fit to={#1}, inner sep=0pt, label=below:{\textcolor{blue}{\emph{#2}}}] {};
			},
		},
		qtree/.style={for tree={parent anchor=south, 
				child anchor=north, align=center, outer sep=-1.1pt}
			},
	}
	\begin{center}
		\begin{forest}, baseline, qtree
			[Fast SVDD 
			[Fast Training, child frame={name=FI}{post-processing}
			[Reduction, child frame={name=S, name=K}{pre-processing}
			[Sampling \\ \cite{alam2020sample, hu2014fast, krawczyk2019instance, li2011selecting, li2018information, li2019health, sun2016heuristic, xiao2014k, zhu2014boundary, qu2019towards}, name=S]
			[Kernel Matrix \\ \cite{scholkopf2000new, achlioptas2002sampling, fine2001efficient, nguyen2008support, williams2001using, yang2012nystrom}, name=K]
			]
			[Optimization \\ \cite{chaudhuri2018sampling, chu2004scaling, kim2007fast, platt1998sequential}, name=D ]
			] 
			[Fast Inference \\ \cite{mika1999kernel, kwok2004pre, bakir2004learning, liu2010fast, peng2012efficient}, name=FI]
			];
		\end{forest}
	\end{center}
	\caption{Categorization of literature on SVDD speedup.}
	\label{fig:rel-work}
\end{figure}
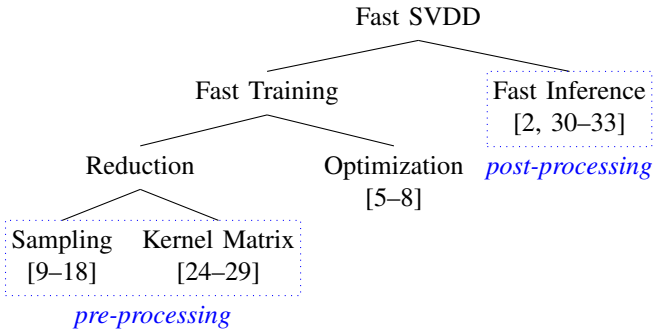

\newcommand{\highlightCategory}{}
\def\highlightCategory|#1#2|{\emph{#1#2}}

SVDD is a quadratic problem (QP).
The time complexity of solving SVDD is in $\mathcal{O}(N^3)$~\cite{chu2004scaling}.
Thus, training does not scale well to large data sets.
However, the time complexity for inference is only linear in the number of support vectors.
So for large $N$, training time is much larger than inference time.
Still, long inference times may be an issue, e.g., in time-critical applications.
So curbing the runtimes has long become an important topic in the SVDD literature.
In \autoref{sec:related-work:categorization}, we categorize existing approaches that focus on SVDD speedup, see \autoref{fig:rel-work} for an overview.
In \autoref{sec:related-work:sampling}, we then turn to \emph{Sampling}, the category our current article belongs to.

\subsection{Categorization}\label{sec:related-work:categorization}
We distinguish between \highlightCategory|Fast Training| and \highlightCategory|Fast Inference|.

\paragraph{Fast Training}
To speed up training of SVDD, one has two options: reduction of the problem size, and optimization of the solver.
For \highlightCategory|Reduction|, one can distinguish further:
A first type reduces the number of observations by \highlightCategory|Sampling|.
This is the category of methods mentioned in our introduction~\cite{alam2020sample, hu2014fast, krawczyk2019instance, li2011selecting, qu2019towards, li2018information, li2019health, sun2016heuristic, xiao2014k, zhu2014boundary}.
A second type reduces the size of the \highlightCategory|Kernel matrix|, e.g., by approximation~\cite{scholkopf2000new, achlioptas2002sampling, fine2001efficient, nguyen2008support}.
Examples are the Nyström-method~\cite{williams2001using} and choosing random Fourier features~\cite{yang2012nystrom}.

\highlightCategory|Optimization| on the other hand decomposes QP into smaller chunks that can be solved efficiently.
Literature features methods that decompose with clustering~\cite{kim2007fast} and with multiple random subsets~\cite{chaudhuri2018sampling}.
The most widely used decomposition methods are sequential minimal optimization (SMO)~\cite{platt1998sequential} and its variants.
These methods iteratively divide SVDD into small QP sub-problems and solve them analytically.
Finally, there is a core-set method that expands the decision boundary by iteratively updating an SVDD solution~\cite{chu2004scaling}.

Reduction and Optimization are orthogonal to each other.
Thus, one can use problem-size reduction in a \emph{pre-processing} step before solving SVDD efficiently.

\paragraph{Fast Inference}
When SVDD uses a non-linear kernel, one cannot compute the pre-image of the center $a$.
Instead, one must compute the distance of an observation to $a$ by a linear combination of the support vectors in the kernel space.
However, literature proposes several approaches to approximate the pre-image of $a$~\cite{mika1999kernel, kwok2004pre, bakir2004learning, liu2010fast, peng2012efficient}.
With this, inference no longer depends on the support vectors, and is in $\mathcal{O}(1)$.
Fast Inference is orthogonal to Fast Training, i.e., it can come as a \emph{post-processing} step, after training.

\subsection{Sampling Methods}\label{sec:related-work:sampling}

\newcommand{\citetable}[1]{\citeauthor{#1}~\cite{#1} & \citeyear{#1}}

\label{sec:sampling}
\begin{table}[t]
\caption{Sampling methods proposed for SVDD.}
\label{tab:rel-work:competitors}
    
	\begin{threeparttable}[b]
    \renewcommand{\arraystretch}{1.1}
    \resizebox{\columnwidth}{!}{%
    \begin{tabular}{l|lclc}
        \toprule
        Method & Publication & Year & Exogenous Parameters\textsuperscript{*}\\
        \midrule
        BPS & \citetable{li2011selecting} & $k {=} \left\lfloor 10\ln N \right\rfloor$, $\varepsilon {=} 0.05$\\
        DAEDS & \citetable{hu2014fast} & $k {=} 30, \, \varepsilon {=} 0.1$, $\delta {=} 0.3$\\
        DBSRSVDD & \citetable{li2019health} & $\textit{minPts} {=} 7$, $\varepsilon {=} 0.5$\\
        FBPE & \citetable{alam2020sample} & $n {=} 360$\\
        HSR & \citetable{sun2016heuristic} & $k {=} 20$, $\varepsilon {=} 0.01 \cdot M$\\
        HSC\textsuperscript{$\dag$} & \citetable{qu2019towards} & $k {=} 20$\\
        IESRSVDD & \citetable{li2018information} & $\varepsilon {=} 0.5$\\
        KFNCBD & \citetable{xiao2014k} & $k {=} 100$, $\varepsilon {=} 0.2$\\
        NDPSR & \citetable{zhu2014boundary} & $k {=} 20$, $\varepsilon {=} 10$\\
        OCSFLSDE\textsuperscript{$\dag$} & \citetable{krawczyk2019instance} & \num{8} different parameters \\
        \bottomrule
    \end{tabular}
    }
	\begin{tablenotes}
		\item * The listed values for the exogenous parameters are the ones used \\ in our experiments.
		\item $\dag$ Not included in our experiments, see~\autoref{sec:experiments-setup} for details.
	\end{tablenotes}
	\end{threeparttable}

\end{table}

Sampling methods take the original data $\mathbf{X}$ set as an input and produce a sample $\mathbf{S}$.
All existing sampling methods assume the \emph{target-only scenario}, i.e., all observations in $\mathbf{X}$ are from the target class.
This is equivalent to a supervised setting where one has knowledge of the ground truth, and $\mathbf{Y} = \langle \text{in}, \text{in}, \dots , \text{in} \rangle$.
Thus, most of the competitors therefore require modifications to apply to the outlier scenario, see \autoref{sec:method:pre-filtering} for details.
In the following, we discuss existing sampling methods for the \emph{target-only scenario}.
We categorize them into different types: \emph{Edge-point} detectors, \emph{Pruning} methods and \emph{Others}.
\autoref{tab:rel-work:competitors} provides an overview.

\paragraph{Edge-point}
Most sampling approaches focus on selecting observations that demarcate $p_{\text{inlier}}$ from $p_{\text{outlier}}$, and therefore are expected to be support vectors.
Such observations are called \enquote{edge points} or \enquote{boundary points}.
Literature proposes different ways to identify edge points.
One idea is to use the angle between an observation and its $k$ nearest neighbors~\cite{li2011selecting, zhu2014boundary} as an indication.
An observation is selected as edge point if most of its neighbors lie within a small, convex cone with the observation as the apex.
One has to specify a threshold for the share of neighbors and the width of the cone~\cite{li2011selecting} as exogenous parameters.
Others suggest to identify edge points through a farthest neighbor search.
For instance, one suggestion is to first sort the observations by decreasing distance to its k-farthest neighbors (KFN)~\cite{xiao2014k}, and then select the top $\varepsilon$ percent as edge points.
The rationale presented in the paper is that inner points are expected to have a lower KFN distance than edge points.
A more recent variant uses angle-based search~\cite{alam2020sample}.
The idea of the paper is to initialize the method by the mean over all observations as the apex and divide the space into a pre-specified number of cones.
For each cone, one only keeps the farthest observation as edge points.

Next, there are methods that select edge points by density-based outlier rankings, e.g., DBSCAN~\cite{li2019health} and LOF~\cite{hu2014fast}.
Here, the assumption is that edge points occur in sparse regions of the data space.
A similar idea is to rank observations with a high distance to all other observations~\cite{li2018information}.
Others have suggested to rank observation highly if they have low density and a large distance to high-density observations~\cite{qu2019towards}.
Naturally, ranking methods require to set a cutoff value to distinguish edge points from other observations.

\paragraph{Pruning}
The idea of pruning is to iteratively remove observations from high-density regions as long as the sample remains \enquote{density-connected}.
One way to achieve this is by pruning all neighbors of an observation closer than a minimum distance, starting from the observation closest to the cluster mean~\cite{sun2016heuristic}.
Yet this approach requires to set the minimum distance threshold, and a good choice is data dependent.

\paragraph{Others}
There is one method that differ significantly from the previous ones~\cite{krawczyk2019instance}.
The basic idea is to generate artificial outliers to transform the problem into a binary classification problem.
Based on the augmented data, one can apply conventional sampling methods such as binary instance reduction.
The sampling method then relies on an evolutionary algorithm where the fitness function is the prediction quality on the augmented data.
Finally, the method only retains the remaining inliers and discards all artificial observations.
However, this requires to solve a large number of SVDD instances in each iteration.

\medskip
To summarize, there are many methods to select a sample for SVDD.
However, they are based upon some intuition regarding the SVDD and do not come with any formal guarantee.
Edge point detectors in particular return a poor sample in some cases, since they do not guarantee coherence of a selected sample, see \autoref{fig:intro-example}.
Further, all existing approaches require to set some exogenous parameter.
But the influence of the parameter values on the sample is difficult to grasp.
Finally, existing sampling methods are designed for the \emph{target-only scenario}.
It is unclear whether they can be modified to work well with the \emph{outlier scenario}.

\section{Density-based Sampling for SVDD}
\label{sec:method}

In this section, we present an efficient and effective sampling method for scaling SVDD to very large data sets.
In a nutshell, we exploit that an SVDD decision boundary is in fact a level-set estimate~\cite{Vert2006-ty}, and that inliers are a super-level set.
The idea behind our sampling method is to remove observations from a data set such that the inlier super-level set does not change.
To this end, we show that \emph{the super-level set of inliers does not change as long as not-selected observations have higher density than the minimum density of selected observations}.
If this \emph{density rule} is violated, sampling may produce \enquote{gaps}, i.e., regions of inliers that become regions of outliers.
Such gaps curb the SVDD quality.
Thus, we strive for a sample of minimal size that satisfies the density rule.

\autoref{fig:idea-method} illustrates our approach.
In a first step, we separate the unlabeled data into outlier and inlier regions based on their empirical density, see~\autoref{sec:method:pre-filtering}.
We then frame sample selection as a optimization problem where the constraints enforce the density rule in \autoref{sec:method:optimization}.
In \autoref{sec:method:rapid} we propose \ac{RAPID}, an efficient and easy-to-implement algorithm to solve the optimization problem.
\OurSamplingMethod returns a small sample which has a close-to-uniform density, i.e., a small sample that still obeys the density rule, and also contains the boundary points of the original data.

\begin{figure}[t]
    \centering
    \includegraphics[width=.9\columnwidth]{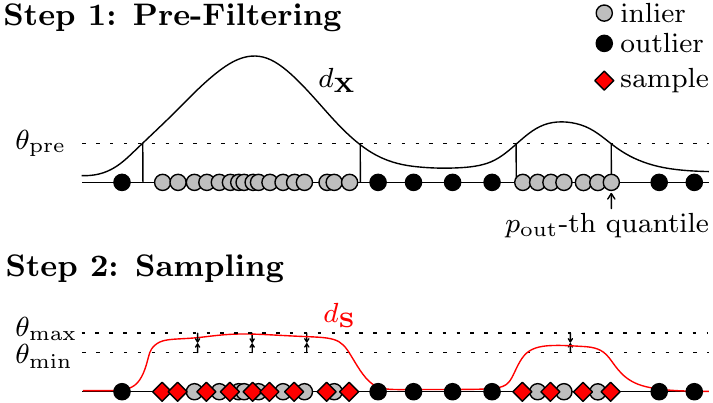}
    \caption{The idea of density-based sampling for SVDD.}
    \label{fig:idea-method}
\end{figure}

\subsection{Density-based Pre-Filtering} \label{sec:method:pre-filtering}

Any sampling method faces an inherent trade-off: reducing the size of the data as much as possible while maintaining a good classification accuracy on the sample.
One can frame this as an optimization problem
\begin{subequations}\label{eq:naive-X}
	\begin{align}
		\underset{\mathbf{S}}{\text{minimize}} \; & \quad \vert \mathbf{S} \vert \tag{\ref*{eq:naive-X}}\\
		\text{subject to} & \quad \textit{diff}(f^\mathbf{S}, f^\mathbf{X}) \leq \varepsilon,\nonumber
	\end{align}
\end{subequations}
where \textit{diff} is a similarity between two decision functions and $\varepsilon$ a tolerable deterioration in accuracy.
Solving \opref{eq:naive-X} requires knowledge of $f^\mathbf{X}$.
But obtaining this knowledge is infeasible.
The reason is that $\vert \mathbf{X} \vert $ is too large to solve --- SVDD would not need any sampling in the first place otherwise.
Thus, one cannot infer which observations $f^\mathbf{X}$ classifies as inlier or outlier.
However, we know that the SVDD hyperparameter $C$ defines a lower bound on the share of observations predicted as outliers in the training data~\cite{tax_support_2004}.
A special case is if $C=1$, since $f^\mathbf{X}(x;\ C\!=\!1) = \text{in}, \forall x \in \mathbf{X}$.
Recall that this is the upper bound of the cost parameter $C$ where SVDD degenerates to a hard-margin classifier, cf.\ \autoref{sec:fundamentals}.
In this case, \textit{diff} is zero if SVDD trained on $\mathbf{S}$, i.e., $f^\mathbf{S}$, also includes all observations within the hypersphere.
Further, we can make use of the following characteristic of SVDD.

\begin{characteristic}[SVDD Level-Set Estimator]\label{characteristic:density}
    SVDD is a consistent level set estimator for the Gaussian kernel~\cite{Vert2006-ty}.
\end{characteristic}
In consequence, inliers form a super-level set with respect to the decision boundary.
Formally, this means that there exists a level set $L_\theta$ and a corresponding level-set classifier $g^\mathbf{X}_\theta$ such that $g^\mathbf{X}_\theta \equiv f^\mathbf{X}$.
We can exploit this characteristic as follows. 
First, we \emph{pre-filter} the data based on their empirical density, such that a share of $p_\text{out}$ observations are outliers.
Formally, $p_\text{out}$ is equivalent to choosing a threshold $\theta_\text{pre}$ on the empirical density, where $\theta_\text{pre}$ is the $p_\text{out}$-th quantile of the empirical density distribution.
Using this threshold in a level-set classifier separates observations into inliers~$\mathbf{I}$ and outliers~$\mathbf{O}$.
\begin{align*}
    \mathbf{I} &= \{x \in \mathbf{X} \colon g_{\theta_\text{pre}}^{\mathbf{X}} = \text{in}\} &
    \mathbf{O} &= \{x \in \mathbf{X} \colon g_{\theta_\text{pre}}^{\mathbf{X}} = \text{out}\} .
\end{align*}

Second, we replace $f^\mathbf{X}$ with $f^\mathbf{I}$ and set $C=1$.
With this, we know that $f^\mathbf{I}(x) = \text{in}, \forall x \in \mathbf{I}$, without training $f^\mathbf{I}$.
Put differently, pre-filtering the data with an explicit threshold allows to get rid of an implicit outlier threshold $C$.
This in turn allows to estimate the level set estimated by SVDD without actually training the classifier.

Pre-filtering does not add any new exogenous parameter, but replaces the SVDD trade-off parameter $C$ with $p_\text{out}$.
Further, $p_{out}$ is a parameter of SVDD, not of our sampling method.
We also deem $p_\text{out}$ slightly more intuitive than $C$, since it makes the lower bound defined by $C$ tight, i.e., pre-filtering assumes an exact outlier ratio of $p_\text{out} = \nicefrac{|\mathbf{O}|}{|\mathbf{X}|}$.
This in turn makes the behavior of SVDD more predictable.
We close the discussion of pre-filtering with two remarks.

\begin{remark}
Technically, one may directly use the level-set classifier $g_{\theta_\text{pre}}^{\mathbf{X}}$ instead of SVDD.
However, inference times are very high, since calculating the kernel density of an unseen observation is in $\mathcal{O}(N)$.
So one would give up fast inference, one of the main benefits of SVDD.
Next, one may be tempted to interpret this pre-filtering step as a way to transform an unsupervised problem into a supervised one to train a binary classifier (e.g., SVM) on $\mathbf{O}$ and $\mathbf{I}$.
However, binary classification assumes the training data to be representative of the underlying distributions. 
This assumption is not met with outlier detection, since outliers may not come from a well-defined distribution.
Thus, binary classification is not applicable.
\end{remark}

\begin{remark}
Pre-filtering is a necessary step with all sampling methods discussed in related work.
In \autoref{sec:related-work}, we have explained that existing sampling methods assume to only have inliers in the data set, i.e., $\mathbf{I} = \mathbf{X}$ and $\mathbf{O} = \emptyset$.
However, if $\mathbf{X}$ contains outliers, this affects the sampling quality negatively and leads to poor SVDD results, see \autoref{sec:experiment-real-world-results}.
\end{remark}

\subsection{Optimal Sample Selection}\label{sec:method:optimization}

After \emph{pre-filtering}, we can reduce \opref{eq:naive-X} to a feasible optimization problem.
We begin by replacing $f^\mathbf{X}$ with $f^\mathbf{I}$.
\begin{subequations}\label{eq:naive-sample}
	\begin{align}
		\underset{\mathbf{S}}{\text{minimize}}\; & \quad \vert \mathbf{S} \vert \tag{\ref*{eq:naive-sample}}\\
		\text{subject to}	& \quad \textit{diff}(f^\mathbf{S}, f^\mathbf{I}) \leq \varepsilon.\nonumber
	\end{align}
\end{subequations}
With \autoref{characteristic:density}, we further know that both classifiers have equivalent level-set classifiers.
We set $g^{\mathbf{I}}_{\theta_\text{pre}}$ as the equivalent level-set classifier for $f^\mathbf{I}$.
For $f^\mathbf{S}$, there also exists a level-set classifier $g^{\mathbf{S}}_{\theta'}$, but the level set $\theta'$ depends on the choice of $\mathbf{S}$.
Thus, we must additionally ensure that $\theta'$ indeed is the level set estimated by training SVDD on $\mathbf{S}$.
The modified optimization problem is
\begin{subequations}\label{eq:level-set-opt}
	\begin{align}
		\underset{\mathbf{S}}{\text{minimize}}\; & \quad \vert \mathbf{S} \vert \tag{\ref*{eq:level-set-opt}} \\
		\text{subject to} & \quad  \textit{diff}(g^{\mathbf{S}}_{\theta'}, g^{\mathbf{I}}_{\theta}) \leq \varepsilon \label{eq:level-set-opt:A}\\
		& \quad g^\mathbf{S}_{\theta'} \equiv f^\mathbf{S}, \label{eq:level-set-opt:B}
	\end{align}
\end{subequations}
where $\equiv$ denotes the equivalence in classifying $\mathbf{S}$.
\constraintref{eq:level-set-opt:B} is necessary, since one may select a sample that yields a level-set classifier similar to the one obtained from $\mathbf{I}$, but on which SVDD returns another decision boundary.
This can, for instance, occur if $\mathbf{S}$ does not contain the boundary points of $\mathbf{I}$.
\opref{eq:level-set-opt} still is very abstract.
We will now elaborate on both of its constraints and show how to reduce them so that the problem becomes practically solvable.

\medskip
\paragraph{\constraintref{eq:level-set-opt:A}}
We now discuss how to obtain a sample that minimizes $\textit{diff}(g^{\mathbf{S}}_{\theta'}, g^{\mathbf{I}}_{\theta})$.
To this end, we use the following theorem.
\begin{theorem}\label{thm:uniform}
    $g^\mathbf{S}_{\theta'} \equiv g^\mathbf{I}_{\theta}$ if $d_{\mathbf{S}}$ is uniform on $\mathbf{I}$.
\end{theorem}
\begin{proof}
    Think of a sample $\mathbf{S} \subseteq \mathbf{I}$ with uniform empirical density $d_\mathbf{S}$.
    Then $\mathbf{S}$ has exactly one level set $\theta'=\theta_\text{min} = \min_{x \in \mathbf{S}} d_{\mathbf{S}}(x)$.
    Further, it also holds that $d_{\mathbf{S}}(x) = \theta_\text{min}$, $\forall x \in \mathbf{I}$.
    It follows that $\min_{x \in \mathbf{I} \setminus \mathbf{S}} d_{\mathbf{S}}(x) = \min_{x \in \mathbf{S}} d_{\mathbf{S}}(x)$, and consequently $g^\mathbf{S}_{\theta_\text{min}}(x) = g^\mathbf{I}_{\theta}(x), \forall x \in \mathbf{I}$.
\end{proof}
\autoref{thm:uniform} implies that one can satisfy \constraintref{eq:level-set-opt:A} with $\varepsilon=0$ if one reduces the sample to one with a uniform empirical distribution $d_\mathbf{S}$.
However, any empirical density estimate on a finite sample can only \emph{approximate} a uniform distribution.
So one should strive for solutions of \opref{eq:level-set-opt} where epsilon is small.
Put differently, one can interpret the difference between a perfect uniform distribution and the empirical density to assess the quality of a sample.
We propose to quantify the fit with a uniform distribution as the difference between the maximum density $\theta_\text{max} = \max_{x \in \mathbf{S}} d_\mathbf{S}(x)$ and minimum density $\theta_\text{min} = \min_{x \in \mathbf{S}} d_\mathbf{S}(x)$:
\begin{equation}
    \Delta^{\mathbf{S}}_\text{fit} = \theta_\text{max} - \theta_\text{min}
\end{equation}
There certainly are other ways to evaluate the goodness of fit between distributions.
However, $\Delta^{\mathbf{S}}_\text{fit}$ has some desirable properties of the sample, which we discuss in \autoref{thm:boundary-removal}.

\newcommand{\optname}{\hyperref[eq:theta-optimization]{SOP}\xspace}

One further consequence of only approximating a uniform density is that there may be some not-selected observations $x \in \mathbf{I}\setminus \mathbf{S}$ with a density value $d_{\mathbf{S}}(x)$ less than $\theta_\text{min}$.
Since the level set estimated by $f^\mathbf{S}$ is $L_{\theta_\text{min}}$, these not-selected observations would be wrongly classified as outliers.
Thus, we must also ensure that $\mathbf{S}$ is selected so that $d_{\mathbf{S}}(x) \geq \theta_\text{min}, \forall x \in \mathbf{I}\setminus \mathbf{S}$.
We can now re-formulate \constraintref{eq:level-set-opt:A} as a sample optimization problem \hyperref[eq:theta-optimization]{\optname}.
\begin{subequations}\label{eq:theta-optimization}
	\begin{align}
	     & \text{\optname}\colon \! \underset{\mathbf{v}, \mathbf{w}, \theta_\text{min}, \theta_\text{max}}{\text{minimize}} \quad \theta_\text{max} - \theta_\text{min} \tag{\ref*{eq:theta-optimization}} \\
		 \text{s.t.} & \underbrace{\sum_{j \in \mathcal{I}} v_j \! \cdot \! k(x_i, x_j)}_{d_\mathbf{S}(x_i)} \geq  \theta_\text{min}, \, \forall i \in \mathcal{I} \label{eq:theta-optimization:min}\\
		 & \sum_{j \in \mathcal{I}} v_j \! \cdot \! k(x_i, x_j) \leq v_i \! \cdot \theta_\text{max}, \, \forall i \in \mathcal{I} \label{eq:theta-optimization:max} \\
    	 & \sum_{j \in \mathcal{I}} w_i \! \cdot \! v_j \! \cdot \! k(x_i, x_j) \leq \theta_{\text{min}}, \, \forall i \in \mathcal{I} \label{eq:theta-optimization:compare} \\
    	 & \sum_{j \in \mathcal{I}} v_j > 0; \sum_{j \in \mathcal{I}} w_j = 1; \; v_j \geq w_j, \forall j \in \mathcal{I}\cup\mathcal{O} \label{eq:theta-optimization:nonzero}\\
		 & v_j = 0, \forall j \in \mathcal{O}; v_j, w_j \in \{0,1\}, \forall j \in \mathcal{I}\cup\mathcal{O} \label{eq:theta-optimization:prefilter}
	\end{align}
\end{subequations}
where $\mathcal{I} = \{i \ \vert \ i \in \{1, \dots, N\}, x_i \in \mathbf{I}\}$, $\mathcal{O}=\{1, \dots, N\} \setminus \mathcal{I}$.
The decision variable $v_j=1$ indicates if an observation $x_j$ is in $\mathbf{S}$, i.e., $\mathbf{S} = \{x_i \in \mathbf{X} \ \vert \ v_i = 1\}$.
If the solution set of \optname is not singular, we select the solution where $\vert \mathbf{S} \vert$ is minimal.
\constraintref{eq:theta-optimization:max} is a technical necessity to obtain the maximum density of $d_\mathbf{S}$.
The first constraint in~\ref{eq:theta-optimization:nonzero} rules out the trivial solution $v = \Vec{0}$.
The first constraint in~\ref{eq:theta-optimization:prefilter} results from the \emph{pre-filtering}, cf.\ \autoref{sec:method:pre-filtering}.

Constraints~\ref{eq:theta-optimization:min}, \ref{eq:theta-optimization:compare}, and \ref{eq:theta-optimization:nonzero} together guarantee that the density of not-selected observations is at least $\theta_\text{min}$, as follows.
Only for one observation $j$ we have $w_j = 1$ and for all other observations $i \neq j, \; w_i = 0$.
Then for Constraint~\ref{eq:theta-optimization:compare} and \ref{eq:theta-optimization:nonzero} to hold, $j$ must be the observation with the minimum density and $d_\mathbf{S}(x_j) = \theta_\text{min}$.
Additionally, with $v_j \geq w_j$ it follows that $v_j = 1$, thus observation $j$ is in the sample $\mathbf{S}$.
So, for any feasible solution of \hyperref[eq:theta-optimization]{\optname} all not-selected observations have a density of at least the minimum density of the selected observations.
From \ref{eq:theta-optimization:min}, it follows that $d_\mathbf{S}(x) \geq \theta_\text{min}, \forall x \in \mathbf{I}$.
So any solution of \hyperref[eq:theta-optimization]{\optname} satisfies \inequalityref{eq:level-set-opt:A} with a small~$\varepsilon$.

\medskip
\paragraph{\constraintref{eq:level-set-opt:B}}

We now show that a solution of \optname also satisfies \constraintref{eq:level-set-opt:B}.
To this end, we make use of the following characteristic.
\begin{characteristic}[Boundary Points]\label{characteristic:boundary-points}
    The set of boundary points are a superset of the support vectors of SVDD.
\end{characteristic}
So for \constraintref{eq:level-set-opt:B} to hold, an optimum of \optname must contain boundary points of $\mathbf{I}$.
We show that a solution with boundary points is preferred over one without boundary points by the following theorem.
\begin{theorem}\label{thm:boundary-removal}
    The set of boundary points does not change when solving \hyperref[eq:theta-optimization]{\optname} iteratively.
\end{theorem}
\begin{proof}
Suppose that there exists a sample $\mathbf{S}$ which is not a local optimum of \hyperref[eq:theta-optimization]{\optname}.
Then there is a boundary point $x_{\text{min}} = \argmin_{x \in \mathbf{S}} d_{\mathbf{S}}(x)$, an observations $x_{\text{max}} = \argmax_{x \in \mathbf{S}} d_{\mathbf{S}}(x)$ and $x_p \in \mathbf{S}$.
Let $\mathbf{S}_p = \mathbf{S} \! \setminus \! \{x_p\}$ and $\mathbf{S}_{\text{max}} = \mathbf{S} \! \setminus \! \{x_{\text{max}}\}$.
If removing $x_p$ from $\mathbf{S}$ is an optimal choice, there must be no other observation that reduces the objective more than $x_p$.
Thus, the following specific case must hold:
\begin{equation}\label{eq:proof2}
    \begin{array}{@{}l@{\;}l}
        & \Delta^{\mathbf{S}_p}_\text{fit} \leq \Delta^{\mathbf{S}_{\text{max}}}_\text{fit} \\[0.5ex]
        \Leftrightarrow & \theta_{\text{max}} \! - \! k(x_p, x_{\text{max}}) \! - \!  (\theta_{\text{min}} \! - \!  k(x_p, x_{\text{min}}) \\
        & \leq \theta_{\text{max}} \! - \!  k(x_{\text{max}}, x_{\text{max}}) \! - \!  (\theta_{\text{min}} \! - \!  k(x_{\text{max}}, x_{\text{min}})) \\[0.5ex]
        \Leftrightarrow & k(x_p, x_{\text{max}}) \! - \!  k(x_p, x_{\text{min}}) \geq 1 \! - \!  k(x_{\text{max}}, x_{\text{min}}).
    \end{array}
\end{equation}
For one, we conclude that $x_p = x_\text{min}$ is not feasible, because in this case the left hand side of \inequalityref{eq:proof2} is strictly negative, and right hand side positive.
Since boundary points have, per \autoref{def:boundary-point}, a density close to $\theta_{\text{min}}$, they cannot be a candidate for removal.

Next, under two assumptions that (A1) the locations of the maximum and of the minimum density are distant from each other, and that (A2) the kernel bandwidth is sufficiently small, we have $k(x_\text{max}, x_\text{min})~\rightarrow~0$, and $k(x_p, x_\text{max}) - k(x_p, x_\text{min}) \geq 1 \Leftrightarrow x_p = x_\text{max}$.
So in this case, removing $x_\text{max}$ is optimal.
From this, it also follows that the minimum density does not change significantly when removing $x_\text{max}$.
With \autoref{def:boundary-point}, it follows that also the set of boundary points does not change after removing $x_\text{max}$.
\end{proof}

\begin{remark}
    Our proof hinges on two assumptions:
    \emph{(A1) A sufficiently large distance between $x_\text{max}$ and $x_\text{min}$.}
    This assumption is intuitive, since removing an observation with a density close to $\max_{x \in \mathbf{S}} d_{\mathbf{S}}(x)$ improves $\Delta_\text{fit}$ more than removing one close to $\min_{x \in \mathbf{S}} d_{\mathbf{S}}(x)$.
    Generally, the distance between $x_\text{max}$ and $x_\text{min}$ depends on the data distribution.
    However, we find that this is not a limitation in practice, see \autoref{sec:experiments}.
    \emph{(A2) A sufficiently small kernel bandwidth.}
    This assumption is reasonable, because when selecting the kernel bandwidth, one strives to avoid underfitting, i.e., avoid kernels bandwidth that are too wide.
    This holds empirically as well, see \autoref{sec:experiments}.
\end{remark}

\begin{algorithm}[t!]
    \setstretch{1.13}
	\small
	\DontPrintSemicolon
	\InOut
	\Input{Data set $\mathbf{X} \in \mathbb{R}^{N \! \times \! M}$
        Kernel function $k(x_i, x_j)$,\\
        Outlier percentage $p_\text{out} \in [0, 1]$}
	\Output{Sample indices $\mathcal{S}$}

	\BlankLine
	$d = \langle \sum_{j=1}^N k(x_1, x_j), \dots, \sum_{j=1}^N k(x_N, x_j) \rangle$
	\Comment*{$\mathcal{O}(N^2)$}
    \Comment{Pre-filtering}
	$\theta_\text{pre} = \text{sort-ascending}(d)_{\lfloor p_\text{out} \cdot N \rfloor}$
	\Comment*{$\mathcal{O}(N\log N)$}
	$\mathcal{S} = \mathcal{I} = \{i \ \vert \ i \in \{1, \dots, N\},\ d_i \geq \theta_\text{pre}\}$
	\Comment*{$\mathcal{O}(N)$}
	$\mathcal{O} = \{i \ \vert \ i \in \{1, \dots, N\}\} \setminus \mathcal{I}$
	\Comment*{$\mathcal{O}(1)$}
	$d = d \! - \!\langle \sum_{j \in \mathcal{O}} k(x_1, x_j), \dots, \sum_{j \in \mathcal{O}} k(x_N, x_j) \rangle$
	\Comment*{$\mathcal{O}(N^2)$}
    \Comment{Sampling}
	\For(\Comment*[f]{$\mathcal{O}(N^2)$}){$\text{iter} \leftarrow 1 \dots |\mathcal{I}|-1$}{
		$r = \argmax_{i \in \mathcal{S}} d_i$\;
		$d = d - \langle k(x_1, x_r), \dots, k(x_N, x_r) \rangle$\;
		$\theta_\text{min} = \min_{i \in \mathcal{S}}\, (d_i)$\;
		\If(){$\exists \, i \in \mathcal{I}: d_i < \theta_\text{min}$}{
			\Return $\mathcal{S}$\;
		}
		$\mathcal{S} = \mathcal{S} \setminus \{r\}$\; 
	}
	\Return $\mathcal{S}$\;
	
	\caption{\OurSamplingMethod}
	\label{alg:our-method}
\end{algorithm}

\hyperref[eq:theta-optimization]{\optname} is theoretically appealing.
However, it is a mixed-integer problem with non-convex constraints, and it is hard to solve.
Thus, solver runtimes quickly become prohibitive, even for relatively small problem instances -- this contradicts the motivation for sampling.
We therefore propose \OurSamplingMethod, a fast heuristic to search for a local optimum of \optname.

\subsection{A RAPID Approximation}\label{sec:method:rapid}

The idea of our approximation is to initialize $\mathbf{S} = \mathbf{I}$, which is a feasible solution to \hyperref[eq:theta-optimization]{\optname}, and remove observations from $\mathbf{S}$ iteratively as long as $\mathbf{S}$ remains feasible, see \autoref{alg:our-method}.

As input parameters \OurSamplingMethod takes the data set $\mathbf{X}$, the expected outlier percentage $p_\text{out}$ and a kernel function $k$.
Lines 1--5 are the initialization and the \emph{pre-filtering}.
\OurSamplingMethod then iteratively selects the most dense observation $x_\text{max}$ in the current sample $\mathbf{S}$ for removal (Line 7) and updates the densities (Line 8).
If $\mathbf{S}\setminus \{x_\text{max}\}$ is infeasible, \OurSamplingMethod terminates (Line 9--11).
Line 10 checks whether there is an observation $x_i \in \mathbf{I}$ that violates \constraintref{eq:theta-optimization:min}.
As required by \hyperref[eq:theta-optimization]{\optname}, \OurSamplingMethod does not remove boundary points.
This is because $x_\text{max}$ must not be a boundary point, as long as $\mathbf{S}$ is not uniform, i.e., $\Delta^{\mathbf{S}}_\text{fit} > 0$.
Thus, a solution of \OurSamplingMethod satisfies both \constraintref{eq:level-set-opt:A} and~\constraintref{eq:level-set-opt:B}.

Finally, we discuss the time complexity of our method.
\begin{lemma}
    The overall time complexity of \OurSamplingMethod is in $\mathcal{O}(N^2)$.
\end{lemma}
\begin{proof}
    See \autoref{alg:our-method} for the step-wise time complexities.
    Overall, the pairwise kernel evaluation dominates the time complexity, and thus \OurSamplingMethod is in $\mathcal{O}(N^2)$.
\end{proof}
So \OurSamplingMethod has a lower time complexity than SVDD.
Further, \OurSamplingMethod is simple to implement with only a few lines of code.
It is efficient, since each iteration (Line 7--11) requires only one pass over the data set to update the densities, compute the new $x_\text{max}$, $\theta_\text{min}$ and minimum inlier density for the termination criterion.
One may further pre-compute the Gram matrix $\mathbf{K}$ for $\mathbf{X}$ to avoid redundant kernel function evaluations.
\section{Experiments}\label{sec:experiments}
\begin{figure*}[t]
	\centering
	\includegraphics[width=0.95\textwidth]{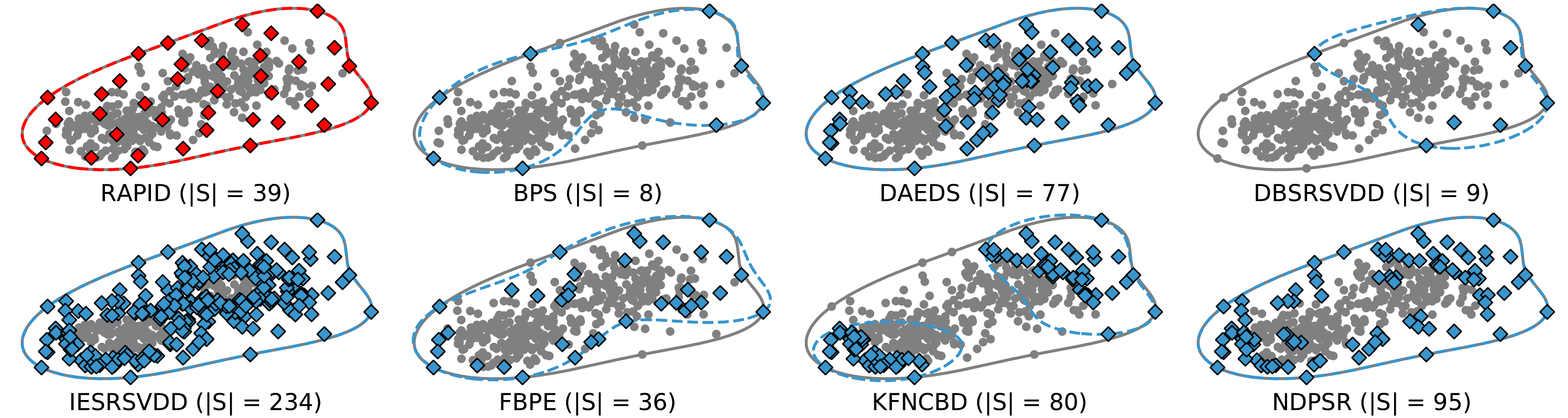}
	\caption{Sampling strategies applied to a synthetic Gaussian mixture with two components and $N = 400$.
	The grey points are the original data set and the red/blue diamonds the selected observations.
	The original decision boundary is the grey line and the red/blue one is the boundary trained on the sample.
	We omit HSR since it returns $\mathbf{S} = \mathbf{X}$ with recommended parameter values.}
	\label{fig:2d-example}
\end{figure*}
 
We now turn to an empirical evaluation of \OurSamplingMethod.
Our evaluation consists of two parts.
In the first part, we evaluate how well \OurSamplingMethod copes with different characteristics of the data, i.e., with the dimensionality, the number of observations, and the complexity of the data distribution, see \autoref{sec:experiment-synthetic-results}.
The second part is an evaluation on a large real-world benchmark for outlier detection.
We have implemented \OurSamplingMethod as well as the competitors in an open-source framework written in Julia~\cite{bezanson2017julia}.
Our implementation, data sets, raw results, and evaluation notebooks are publicly available.
\footnote{\website\label{fn:website}}

\begin{figure*}[ht!]
	\hspace{1.4em}
	\subfloat[Scaling N (M=\num{50}, \#Components = 5)]{\label{fig:eval:syn:size}\hspace{15em}}\hspace{1.2em}
	\subfloat[Scaling M (N=\num{1000}, \#Components = 5)]{\label{fig:eval:syn:dimensionality}\hspace{15em}}\hspace{0.8em}
	\subfloat[Scaling \#Components (N=\num{1000}, M=\num{50})]{\label{fig:eval:syn:difficulty}\hspace{15em}}\\
	\includegraphics[width=0.95\textwidth]{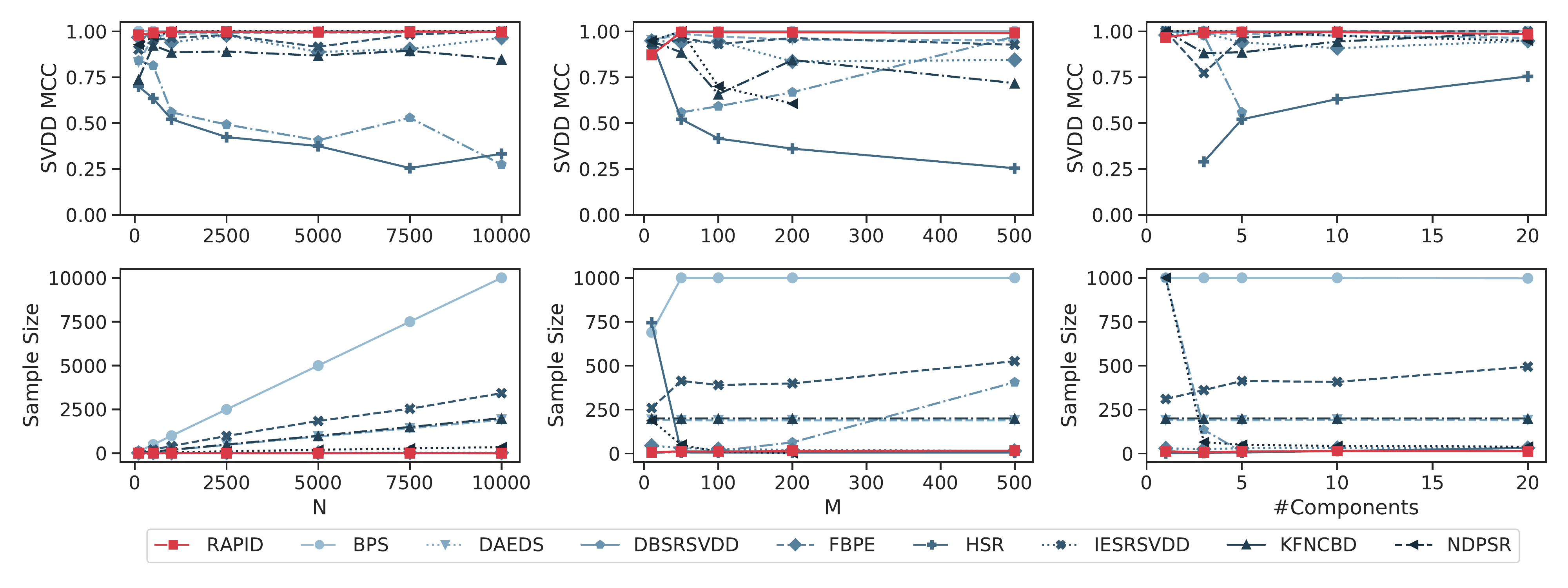}
	\caption{Evaluation on synthetic data with varying data size (N), dimensionality (M), and complexity (\#Components).
		An optimal sampling always yields a MCC of 1 in the upper row and very small sample size in the bottom row, i.e., altering any data characteristics does not influence the sampling.
	Some values for the competitors are missing due to an empty sample.}
	\label{fig:eval:syn}
\end{figure*}

\subsection{Setup}\label{sec:experiments-setup}
We first introduce our experimental setup, including evaluation metrics, as well as the parametrization of SVDD and its competitors.
Recall that \OurSamplingMethod does not have any exogenous parameter.
One must only specify $p_\text{out}$ instead of the SVDD hyperparameter $C$, cf.\ \autoref{sec:method:pre-filtering}.

\paragraph{Metrics}
Sampling methods trade classification quality for sample size, and one must evaluate this trade-off explicitly.
We report the sample size $\vert \mathbf{S}\vert$ and sample ratio $\nicefrac{\vert \mathbf{S}\vert}{\vert \mathbf{X} \vert}$ for each result.
For classification quality, we use the Matthews Correlation Coefficient (MCC) on $\mathbf{X}$.
MCC is well-suited for imbalanced data and returns values in $[-1, 1]$; higher values are better.
Our experiments do not require a train-test split, since all sampling methods are unsupervised.
For non-deterministic methods, we report average values over five repetitions.
Our experiments ran on an AMD Ryzen Threadripper 2990WX with \num{64} virtual cores and \num{128}~GB~RAM.

\paragraph{SVDD}
SVDD requires to set two hyperparameters: the Gaussian kernel parameter $\gamma$ and the trade-off parameter $C$.
We tune $\gamma$ with \emph{Scott's Rule}~\cite{scott2015multivariate} for real-world data. 
For high-dimensional synthetic data, however, we found that the \emph{Modified Mean Criterion}~\cite{liao2018new} is a better choice.
Because of \emph{pre-filtering} we set $C=1$, cf.\ \autoref{sec:method:pre-filtering}.

\paragraph{Competitors}
We compare our method against \num{8} competitors, see \autoref{tab:rel-work:competitors}.
The approaches from~\cite{qu2019towards} and~\cite{krawczyk2019instance} require to solve several hundreds of SVDDs, resulting in prohibitive runtimes.
We do not include them in our evaluation.
We initialize the exogenous parameters according to the guidelines in the original publications.
In some cases, the recommendations do not lead to a useful sample, e.g., $\mathbf{S} = \emptyset$.
To ensure a fair comparison, we mitigate these issues by fine-tuning the parameter values through preliminary experiments.

Next, we compare two variants of each competitor: sampling on $\mathbf{X}$ as in their original version, and sampling on $\mathbf{I}$, i.e., after applying our \emph{pre-filtering}.
The \emph{pre-filtering} requires to specify the expected outlier percentage.
In practice, one can rely on domain knowledge or estimate it~\cite{achtert2010visual}.
To avoid any bias when over- or under-estimating the outlier percentage, we set it to the true percentage.
Nevertheless, we have run additional experiments where we deliberately deviate from the true percentage.
We found that deviating affects the performance of all sampling methods similarly.
So, our conclusions do not depend on this variation, and we report the respective results only in the supplementary materials.\textsuperscript{\ref{fn:website}}

We also evaluate against random baselines.
Each baseline $\textit{Rand}_\textit{r}$ returns a random subset with a specified sample ratio~$r$.
We report results for a range of sample ratios $r \in [0.01, 1.0]$ to put the quality of competitors into perspective.

\subsection{Evaluation of Sample Characteristics}\label{sec:experiment-synthetic-results}

The first part of our experiments validates different properties of \OurSamplingMethod and of its competitors.
Our intention is to give an intuition of how a sample is selected, and to explore under which conditions the sampling methods work well.
The basis for our experiments are synthetic data sets with controlled characteristics.
Specifically, we generate data from Gaussian mixtures with varying number of mixture components, data dimensions, and number of observations.
We run these experiments to answer the following two questions.

\smallskip
\begin{enumerate}[label=\textbf{Q\arabic*}, leftmargin=2.5em, rightmargin=2em, topsep=0.4ex, parsep=0.4ex, ref=Question~Q\arabic*]
	\item How are observations in a sample distributed? \label{q:sample-distribution}
\end{enumerate}
To get an intuition about the sample distribution, we run \OurSamplingMethod and the competitors on a bi-modal Gaussian mixture, see \autoref{fig:2d-example}.
The tendencies of the methods to select boundary points and inner points are clearly visible.
For instance, BPS only selects a sparse set of boundary points; IESRSVDD only prunes high-density areas.
As expected, \OurSamplingMethod selects both the boundary points and a uniformly distributed set of inner points.
The decision boundary of \OurSamplingMethod matches the one obtained from the full data set perfectly.
Only three competitors (DAEDS, IESRSVDD, and NDPSR) also result in an accurate decision boundary.
But all of them produce significantly larger sample sizes than \OurSamplingMethod.

\smallskip
\begin{enumerate}[resume*]
	\item To what extent do data characteristics influence a sample and the resulting classification quality? \label{q:data-characteristics}
\end{enumerate}
To explore this question, we individually vary the number of observations, the dimensionality, and the number of the mixture components.

\emph{Number of observations}:
Ceteris paribus, increasing the number of observations should not have a significant impact on the observations selected.
This expectation is reasonable, since increasing the data size does not change the underlying distribution and the true decision boundary.
\autoref{fig:eval:syn:size} graphs the sample quality and sample size for the different methods.
Many competitors (BPS, IESRSVDD, KFNCB, and DAEDS) do not scale well with more observations, i.e., the sample sizes increase significantly.
BPS scales worst and only removes a tiny fraction of observations.
Further, the sample quality drops significantly with more than \num{500} observations for some competitors (DBSRSVDD and HSR).
\OurSamplingMethod on the other hand is robust with increasing data size, for both sample quality and sample size.
The sample sizes returned are small, even for large data sets, and the resulting quality is always close to \mbox{MCC = \num{1.0}}.

\begin{figure}[t!]
	\centering
	\includegraphics[width=\columnwidth, trim= 0 10 0 0]{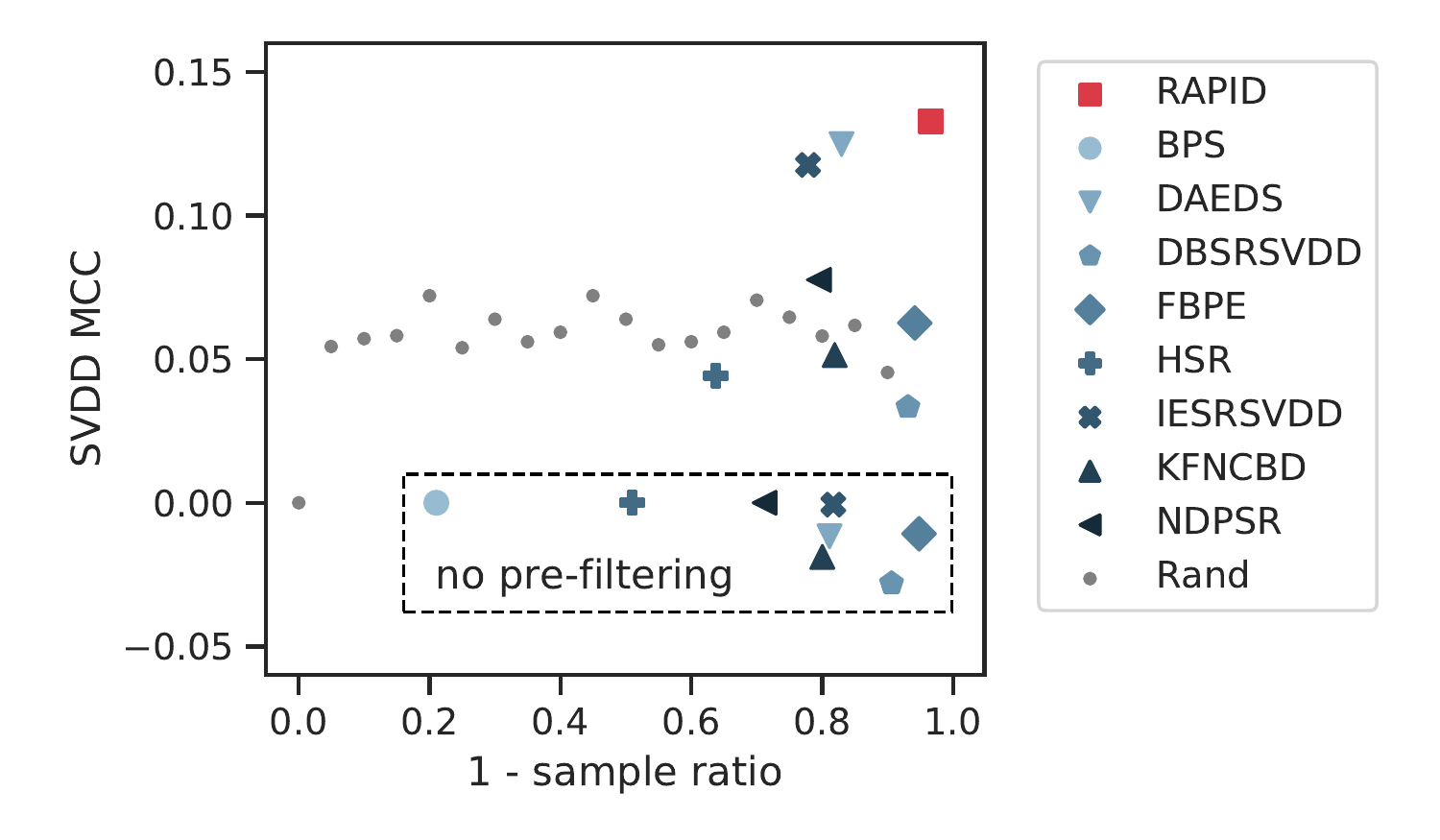}
	\caption{Median MCC and ratio of observations removed by sampling (1 - sample ratio =  $\nicefrac{(N - \vert\mathbf{S}\vert)}{\vert \mathbf{X}\vert}$) over real-world data.\textsuperscript{\ref{fn:remark-results}}
	Rand is shown for different $r \in [0.01, 1.0]$.
	BPS with pre-filtering did not solve for large data sets.}
	\label{fig:eval:real-world}
\end{figure}

\stepcounter{footnote}
\footnotetext{Because of limited space, we report median statistics, but results also hold for mean values and individual comparisons (ranks), see \website.\label{fn:remark-results}}

\emph{Dimensionality}:
The expectation is that the sample quality does not deteriorate with increasing dimensionality.
However, sample sizes may increase slightly.
This is because determining a decision boundary of a high-dimensional manifold requires more observations than of a low-dimensional one.
\autoref{fig:eval:syn:dimensionality} shows the sample quality and size.
For some competitors (HSR, NDPSR, and KFNCBD), sample quality decreases with increasing dimensionality.
This indicates that they do not select observations in all regions.
This in turn leads to misclassification.
Even tuning exogenous parameter values does not mitigate these effects.
As desired, \OurSamplingMethod returns a small sample in all cases, with high classification accuracy.

\emph{Number of Mixture Components}:
Finally, we make the data set more difficult by increasing the number of Gaussian mixture components.
Like before, we expect sample sizes to increase slightly, since the generated manifolds are more difficult to classify.
\autoref{fig:eval:syn:difficulty} shows the sample quality and size.
For HSR and DBSRSVDD, sampling quality fluctuates significantly.
NDPSR and DBSRSVDD do not prune any observation with only one component.
We think that these effects are due to the sensitivity to the exogenous parameters of the various methods.
This is, methods with fluctuating results would require different parameter values for data sets of different difficulties.
However, the competitors do not come with a systematic way to choose parameter values to adapt to varying data set difficulty.
\OurSamplingMethod in turn is very robust to changes in difficulty.
As expected, the sample size increases only slightly with increasing difficulty.
The classification accuracy is close to \mbox{MCC = \num{1.0}}, even for high difficulties.

\smallskip
\textit{
In summary, our experiments on synthetic data reveal that many competitors are sensitive to data size, dimensionality, and complexity.
Different parameter values may mitigate the effects in a few cases, but selecting good values is difficult.
\OurSamplingMethod on the other hand is very robust.
It adapts well to different data characteristics and does not require any parameter tuning.
}

\subsection{Benchmark on Real-World Data}\label{sec:experiment-real-world-results}

Next, we turn to data sets with real distributions and more diverse data characteristics.
The basis for our experiments are \num{21} standard benchmark data sets for outlier detection~\cite{campos_evaluation_2016}.
Campos et al. constructed this benchmark from classification data where one of the classes is downsampled and labeled as outlier.
The data sets have different sizes (\numrange{80}{49534} observations), dimensionality (\numrange{3}{1555} dimensions) and outlier ratios (\SIrange{0.2}{75.38}{\percent}, median \SI{9.12}{\percent}).
Again, we structure our experiments along two questions.

\begin{table}[t]
\caption{Median metrics over real-world data.\textsuperscript{\ref{fn:remark-results}}}
	\label{tbl:eval:add-stats}
	\centering
    \begin{threeparttable}[b]
    \resizebox{0.95\columnwidth}{!}{%
	\begin{tabular}{lrrr|rr|c}
		\toprule
		& \multicolumn{3}{c}{runtimes} & \multicolumn{2}{c}{sample} & quality \\
		{} &  $t_\text{samp}$ &  $t_\text{train}$ & $t_\text{inf}\textsuperscript{$*$}$ & size &   ratio &  MCC \\
		\midrule
		RAPID       &    0.02 &     0.02 &    0.01 &    21 &   0.03 &  0.13 \\
		\midrule
		BPS$^\dag$      &    0.61 &     0.47 &    0.15 &   385 &   0.60 &  \multicolumn{1}{c}{$\dag$} \\
		DAEDS    &    0.59 &     0.04 &    0.07 &    98 &   0.17 &  0.12 \\
		DBSRSVDD &    0.01 &     0.02 &    0.01 &    40 &   0.07 &  0.03 \\
		FBPE     &    0.07 &     0.02 &    0.01 &    39 &   0.06 &  0.06 \\
		HSR      &    0.47 &     0.07 &    0.03 &   130 &   0.36 &  0.04 \\
		IESRSVDD &    0.02 &     0.10 &    0.03 &   154 &   0.22 &  0.12 \\
		KFNCBD   &    0.53 &     0.06 &    0.04 &   100 &   0.18 &  0.05 \\
		NDPSR    &    0.51 &     0.07 &    0.03 &   103 &   0.21 &  0.08 \\
		\bottomrule
	\end{tabular}
	}
	\begin{tablenotes}
	\item $\text{*}$ time for inference in seconds per 1000 observations.
	\item $\dag$ did not solve for large data sets.
	\end{tablenotes}
    \end{threeparttable}
\end{table}

\smallskip
\begin{enumerate}[resume*]
	\item How well do methods adapt to real-world data sets? \label{q:benchmark}
\end{enumerate}
First, we compare \OurSamplingMethod against competitors without any pre-processing.
\autoref{fig:eval:real-world} plots the median sample ratio against the SVDD quality over all data sets.\textsuperscript{\ref{fn:remark-results}}
Good sampling methods return small sample ratios and yield high SVDD quality, i.e., they appear in the upper right corner of the plot.
All of the competitors in their original version, i.e., without pre-filtering, result in poor SVDD quality, much lower than the Rand baselines.
The reason is that they expect all observations to be inliers.

With our \emph{pre-filtering}, SVDD qualities of competitors improve considerably, see \autoref{fig:eval:real-world} and \autoref{tbl:eval:add-stats}.
Still, \OurSamplingMethod outperforms its competitors; none of them produces a sample with higher SVDD quality or smaller sample size than \OurSamplingMethod.
The methods closest to \OurSamplingMethod are DAEDS and IESRSVDD, with similar SVDD quality, but significantly larger sample sizes.
On average, the sample selected by \OurSamplingMethod even yields the same quality as training a SVDD without sampling.\textsuperscript{\ref{fn:original-data}}

\smallskip
\begin{enumerate}[resume*]
	\item What are the runtime benefits of sampling? \label{q:runtimes}
\end{enumerate}
Finally, we look at the impact of sampling on algorithm runtimes, see \autoref{tbl:eval:add-stats}.
We measure the execution runtimes of the sampling method ($t_\text{samp}$), of SVDD training on the sample ($t_\text{train}$), and of the inference ($t_\text{inf}$).
Overall, all methods have reasonable runtimes for sampling, with BPS being the slowest with \SI{0.61}{\second} on average.
However, \OurSamplingMethod is the fastest method overall.
Methods with runtimes similar to \OurSamplingMethod, such as DBSRSVDD, feature significantly lower SVDD quality.
Compared to SVDD applied to large original data sets without sampling, \OurSamplingMethod reduces training times from over one hour to only a few seconds.\textsuperscript{\ref{fn:original-data}}

\smallskip
\textit{
In summary, \OurSamplingMethod outperforms its competitors on real-world data as well.
There is no other method with higher SVDD quality and similarly small sample sizes.
\OurSamplingMethod scales very well to very large data sets and reduces overall runtimes by up to an order of magnitude.
}

\stepcounter{footnote}
\footnotetext{Based on data sets with non-prohibitive runtime, i.e., $N < \num{25000}$, see \website for details.\label{fn:original-data}}
\section{Conclusions}
\label{sec:conclusions}

SVDD does not scale well to large data sets due to long training runtimes.
Therefore, working with a sample instead of the original data has received much attention in the literature.
Various existing sampling approaches guess the support vectors of the original SVDD solution from data characteristics.
These methods are difficult to calibrate because of unintuitive exogenous parameters.
They also tend to perform poorly regarding outlier detection.
One reason is that including support vector candidates in the sample does not guarantee them to indeed become support vectors.

Our article addresses these issues. 
We formalize SVDD sample selection as an optimization problem, where constraints guarantee that SVDD indeed yields the correct decision boundaries.
We achieve this by reducing SVDD to a density-based decision problem, which gives way to rigorous arguments why a sample indeed retains the decision boundary.
To solve this problem effectively, we propose a novel iterative algorithm \OurSamplingMethod.
\OurSamplingMethod does not rely on any parameter tuning beyond the one already required by SVDD.
It is efficient and consistently produces a small high-quality sample.
Experiments show that the way we have framed sampling as an optimization problem improves substantially on existing methods with respect to runtimes, sample sizes, and classification accuracy.

\bibliographystyle{IEEEtranN}
\bibliography{IEEEabrv,ocs}

\end{document}